\documentclass[letterpaper, 10 pt, journal, twoside]{IEEEtran}

\IEEEoverridecommandlockouts                              

\let\labelindent\relax
\usepackage{array}
\usepackage{footmisc}
\usepackage{graphicx}
\usepackage{cite}
\usepackage{amsmath}
\usepackage{amssymb}
\usepackage{booktabs}
\usepackage{xcolor}
\usepackage{mathtools}
\usepackage{pgfplots}
\usepgfplotslibrary{groupplots}
\pgfplotsset{compat=newest}
\usepackage{subcaption}
\usepackage{algorithm}
\usepackage[noend]{algpseudocode}
\usepackage[shortlabels]{enumitem}
\algrenewcommand\alglinenumber[1]{\footnotesize #1}
\usetikzlibrary{patterns}
\tikzstyle{every node}=[font=\small]
\usepackage{wrapfig}
\pgfplotsset{compat=1.11,
    /pgfplots/ybar legend/.style={
    /pgfplots/legend image code/.code={%
       \draw[##1,/tikz/.cd,yshift=-0.25em]
        (0cm,0cm) rectangle (3pt,0.8em);},
   },
}

\makeatletter
\let\NAT@parse\undefined
\makeatother
\usepackage{url} 
\usepackage{amsthm} 

\newcommand\mydots{\hbox to 1em{.\hss.\hss.}}
\DeclareMathOperator{\initialstate}{\mathcal{I}}

\DeclareMathOperator{\goalstate}{\mathcal{G}}
    \newtheorem{theorem}{Theorem}[section]

\begin{document}

\title{\LARGE \bf
Learning to Search in Task and Motion Planning with Streams
}

\author{
Mohamed Khodeir$^{*\dag}$
\quad
Ben Agro$^{*\dag}$
\quad 
Florian Shkurti$^{\dag}$
\thanks{Manuscript received: August 23, 2022; Revised: November 26, 2022; Accepted: January 16, 2023. }
\thanks{ This paper was recommended for publication by Editor Hanna Kurniawati upon evaluation of the Associate Editor and Reviewers’ comments.}
\thanks{$^*$Authors contributed equally.}
\thanks{$^{\dag}$Robot Vision and Learning Laboratory, University of Toronto Robotics Institute. 
{\tt\small \{m.khodeir, ben.agro\}@mail.utoronto.ca, florian@cs.toronto.edu}}
\thanks{Digital Object Identifier (DOI): 10.1109/LRA.2023.3242201}
}

\markboth{IEEE Robotics and Automation Letters. Preprint Version. Accepted January, 2023}{Khodeir et al.: Learning To Search in TAMP with Streams} 

\maketitle


\IEEEpeerreviewmaketitle

\begin{abstract}
Task and motion planning problems in robotics combine symbolic planning over discrete task variables with motion optimization over continuous state and action variables. Recent works such as PDDLStream~\cite{garrett2020pddlstream} have focused on optimistic planning with an incrementally growing set of objects until a feasible trajectory is found. However, this set is exhaustively expanded in a breadth-first manner, regardless of the logical and geometric structure of the problem at hand, which makes long-horizon reasoning with large numbers of objects prohibitively time-consuming. To address this issue, we propose a geometrically informed symbolic planner that expands the set of objects and facts in a best-first manner, prioritized by a Graph Neural Network that is learned from prior search computations. We evaluate our approach on a diverse set of problems and demonstrate an improved ability to plan in difficult scenarios. We also apply our algorithm on a 7DOF robotic arm in block-stacking manipulation tasks.     
\end{abstract}



    


\section{Introduction}
\label{sec:intro}
Task and motion planning (TAMP) involves searching over both symbolic actions that determine a high-level task sequence and low-level motions that result in feasible trajectories \cite{garrett2020integrated, hierarchical_tamp_now, toussaint_stable_modes}. The symbolic task planner operates on high-level abstractions of the environment, including object definitions, operations that can be applied to them, and pre- and post-conditions that these operations have to satisfy. The motion planner requires a non-abstracted description of the environment, and is responsible for informing the symbolic planner about the kinematic and dynamic feasibility of a proposed coarse task plan, possibly leading to backtracking. This interplay between high- and low-level planning has a significant effect on the total runtime required to discover a solution. 

PDDL \cite{ghallab1998pddl} is a widely used representation for symbolic planners, but requires a pre-discretized state space and does not allow for continuous variables.
PDDLStream~\cite{garrett2020pddlstream} extends PDDL by introducing \textit{streams}, conditional generators that model black-box sampling procedures. Streams enable continuous quantities, such as poses and configurations to be created dynamically, allowing planning without the need for pre-discretization. Existing PDDLStream algorithms use streams to iteratively add more of these quantities to the planning problem until a feasible trajectory can be found, leveraging existing PDDL planners as a subroutine. However, this process can induce very large planning problems, especially in cases that involve long-horizon reasoning with many objects.

\begin{figure}[h]
     \centering
     \begin{subfigure}[b]{0.49\linewidth}
         \centering
         \includegraphics[width=\textwidth]{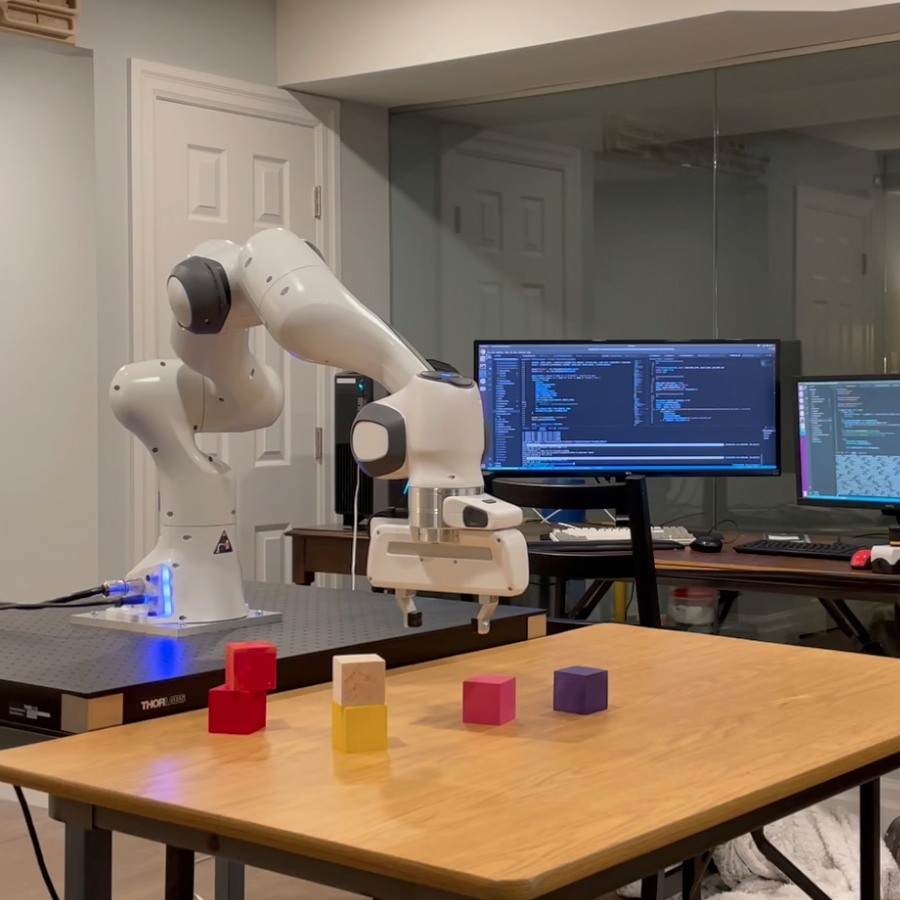}
         \label{fig:stacking_before}
     \end{subfigure}
     \begin{subfigure}[b]{0.49\linewidth}
         \centering
         \includegraphics[width=\textwidth]{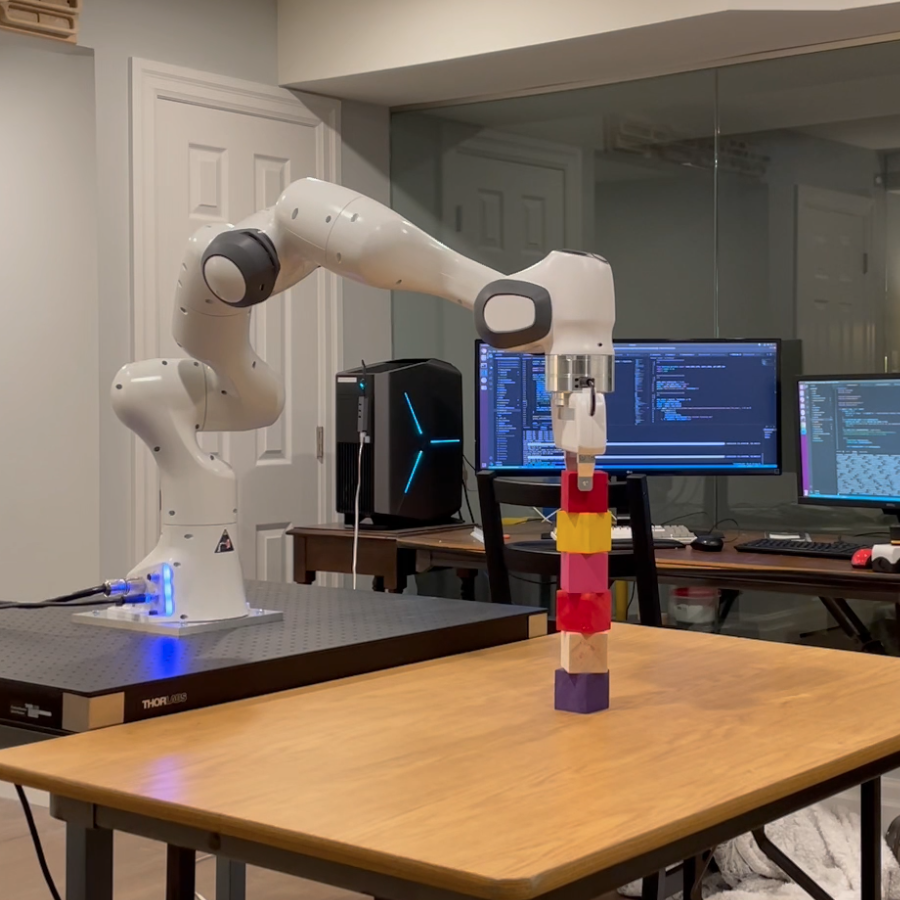}
         \label{fig:stacking_after}
     \end{subfigure}
     \caption{Our planner solves the six-block stacking problem with a 7DOF robot arm\protect\footnotemark. Existing TAMP methods struggle to find a solution for more than four blocks (see Table \ref{tab:stacking_height}). By learning from easier intermediate planning problems, we show improved performance and run time for larger scenes. }
     \label{fig:real_stacking}
\end{figure}
\footnotetext{\label{fnlabel} Code and video can be found on the project webpage: \url{https://rvl.cs.toronto.edu/learning-based-tamp/}}
To address this shortcoming, we present a learning-based method to guide a best-first search in the space of streams. Our main contributions are (a) a Graph Neural Network (GNN) architecture that scores the relevance of streams, and (b) a queue-based algorithm that prioritizes their addition to the planning problem. The model is trained on prior planning experience from solving similar problem instances, drawn from a training distribution of environments. We compare our algorithm (\texttt{INFORMED}) across several simulated environments against the best performing PDDLStream algorithm, as well as a prior work which uses a learned GNN to predict object importance. Our experiments show that \texttt{INFORMED} can outperform baselines on previously unseen problems, finding feasible plans faster, including on problems of larger size, where the baseline fails to find solution. 

\section{Related Work}
\label{sec:Related} 

\textbf{Integrated task and motion planning.} There is a vast literature on the problem of integrating the geometric reasoning required by motion planning with the symbolic reasoning that is necessary for planning to achieve abstract goals; see \cite{garrett2020integrated} for a detailed taxonomy. 
Our work builds on the PDDLStream \cite{garrett2020pddlstream} formalism, which we introduce in more detail in the background section. 
Several algorithms for PDDLStream problems have since been proposed, including \cite{https://doi.org/10.48550/arxiv.2103.05456} which uses Monte-Carlo Tree Search to efficiently explore the space of plans in search of a low cost solution. 
PDDLStream has also been used to facilitate belief-space planning in partially observed environments \cite{9196681}.

There is a long history of prior research, including~\cite{cambon, plaku2010sampling}, which combine symbolic planners with complete geometric planner.
The need for selecting the right hierarchical abstractions for symbolic planning and favoring feasibility and real-time results over optimality was emphasized in~\cite{hierarchical_tamp_now}. Logic Geometric Programming combined symbolic planning and trajectory optimization~\cite{toussaint2015logic, toussaint_stable_modes}, even for dynamic physical motions involving tool use, while~\cite{z3_tamp} integrated sampling procedures with SAT solvers as symbolic planners.    
\nocite{hauser2009integrating}


\textbf{Learning for TAMP}. Motivated by the success of learning in the context of robotics, recent work has sought to combine the ability of TAMP systems to plan for novel temporally extended goals with learning methods. 
There are several avenues of research under this umbrella: methods which learn continuous action samplers for capabilities that may be difficult to engineer (e.g. pouring) \cite{wang2020learning, kim2018guiding}, those which learn the symbolic representations with which to plan \cite{silver2021learning, loula2020learning, diehl2021automated}, those that integrate perception learning and scene understanding into TAMP~\cite{deep_pddl_tamp_policies, Zhu2020HierarchicalPF}, and those which attempt to learn search guidance from experience \cite{driess2020deep, driess2020deeph,beomjoon2020learning, https://doi.org/10.48550/arxiv.2203.10568, kim2019learning}.
In contrast to these works, we seek to learn a heuristic to guide the construction of an optimistic planning problem with a minimal set of facts. These represent contraints on the continuous parameters of a potential solution, and so our work can be seen as a generalization of \cite{kim2019learning}, which similarly learns to select helpful constraints from a discrete set. We leverage an off-the-shelf domain-independent search sub-routine, thus our work can be synergistically combined with methods like \cite{driess2020deep, driess2020deeph,  beomjoon2020learning} that learn a domain-specific heuristic. 

\textbf{Planning with many objects.} The motivation of our work shares similarities with \cite{silver2020planning} and \cite{gnad2019learning}, which attempt to speed up task planning by learning models which exclude irrelevant objects or actions (respectively) from the initial problem representation. Our work can be seen as an extension of these methods.




\section{Background}
\label{sec:Setup}

A PDDLStream\cite{garrett2020pddlstream} problem $(\mathcal P, \mathcal A, \mathcal S, \mathcal O, \initialstate, \goalstate)$ is given by a set of predicates $\mathcal P$, actions $\mathcal A$, streams $\mathcal S$, initial objects $\mathcal O$, an initial state $\initialstate$, and a goal state $\goalstate$. Similarly to PDDL, a \textit{predicate}, $p$ is a boolean function. An instance of a predicate $p(\bar{x})$ applied on a tuple of objects $\bar{x} = \langle x_1, \dots, x_n \rangle$ is called a  \textit{fact}. PDDL \textit{objects} are references to entities in the planning problem. These can refer to anything from the identities of physical objects, to continuous quantities or even complex data structures such as those representing trajectories.
A \textit{state} is a set of facts.  $\initialstate$ is an example of a state. An action $a$ is a tuple of three components: parameters $\bar{X}$, a set of preconditions $\text{pre}(a)$, and a set of effects $\text{eff}(a)$. An action instance $a(\bar{x})$, whose parameters $\bar{X}$ have been assigned to particular objects $\bar{x}$, is only applicable/available in a given state if it satisfies the set of facts in the action's preconditions  $\text{pre}(a(\bar{x}))$. The set of facts $\text{eff}(a(\bar{x}))$ specify how a state will change after $a(\bar{x})$ is applied. A solution to the planning problem is a sequence of valid action instances $\pi = [a_1(\bar{x}_1), \dots a_N(\bar{x}_N)]$ that transform $\initialstate$ into $\goalstate$. A helpful concept is the \textit{preimage} of a plan; the set of facts that must hold in order for the entire plan to be applicable in a given initial state. This is defined as follows: 
$$\text{PREIMAGE}(\pi) = \cup_{i = 1}^N \big( \text{pre}(a_i(\bar{x}_i)) - \cup_{j < i} \text{eff}(a_j(\bar{x}_j) \big).$$

\textbf{Example.} To describe the environment depicted in Fig.~\ref{fig:model}, we might define predicates such as \texttt{on-block(?b1 ?b2)} and \texttt{at(?b ?$X_{WB}$)}. We could then define an action that stacks one block onto another as \texttt{stack(?q ?b ?lb ?$X_{WL}$)} with preconditions that include \texttt{clear(?lb)} and \texttt{at(?lb ?$X_{WL}$)}, requiring that the lower block is clear and at the given world pose $X_{WL}$. One effect of this action might be \texttt{on-block(?b  ?lb)}, but there may be other side-effects such as a new configuration of the robot arm. To describe a particular stacking problem in this environment, we simply need to specify objects $\mathcal{O} = \{b_1, b_0, b_2, a, s_0, s_1, q_0, \dots\}$, initial state $\initialstate = \{\texttt{ on-block}(b_1 \, b_0), \texttt{at}(b_0 \, X_{Wb_0}), \dots\}$, and goal state $\goalstate = \{\texttt{on-surface}(b_0 \, s_1), \texttt{on-surface}(b_1 \, s_1)\}$.

\textbf{Streams.} The set of streams, $\mathcal S$, distinguishes a PDDLStream problem from traditional PDDL. Streams are conditional generators which yield objects that satisfy specific constraints conditioned on their inputs. Formally, a stream, $s$, consists of input parameters $s.input$, a set of domain predicates $s.domain$, output parameters $s.output = \bar{o}$, and a set of certified predicates $s.certified$. $s.domain$ is the set of predicates that must evaluate to true for an input tuple $s.input = \bar{x}$ to be valid. This ensures the correct types of objects (e.g., configurations, poses etc.) are provided to the generators. $s.certified$ are predicates on $s.input$ and $s.output$ that assert facts that $\langle \bar{x}, \bar{o}\rangle$ always satisfy.

\textbf{Example.} In our block stacking environment (Fig. \ref{fig:model}), we can define a stream \texttt{find-place(?b, ?s) $\to$ ?$X_{WB}$} which takes as input a block \texttt{?b} and a surface \texttt{?s}, and produces a world pose \texttt{?$X_{WB}$}, at which \texttt{?b} would be on \texttt{?s}. The \texttt{domain} facts of this stream would restrict the types of its input objects (i.e. \texttt{find-place.domain} = \{\texttt{block(?b)}, \texttt{surface(?s)}\}), and the set of certified facts might contain the single fact \texttt{block-support(?block ?surface ?$X_{WB}$)}. Given such a stream, a planner could dynamically request placement poses for any object/surface pair.


Streams can be applied recursively, to generate a potentially infinite set of objects and their associated facts, starting from those in $\initialstate$ and $\mathcal O$, respectively. A crucial aspect of PDDLStream problems is that there usually does not exist a plan whose \textit{preimage} is a subset of $\initialstate$. Instead, solvers of PDDLStream problems must use the streams in $\mathcal S$ to recursively expand $\initialstate$ to $I_{opt}$ such that there exists a plan $\pi$ whose \textit{preimage} is a subset of $I_{opt}$. This is reminiscent of theorem proving problems that start from a set of axioms and use rules of inference to expand the set of known theorems. 

We use the term stream \textit{instantiation} to refer to the act of creating a new instance of a stream on a specific object tuple. We distinguish this from \textit{evaluation} or \textit{sampling} of a stream instance, which refers to the act of invoking the associated black-box sampler on the given inputs. When a stream instance is instantiated, its output objects (and associated certified facts) are said to be \textit{optimistic} - they are symbolic referents that may or may not have possible real values. In contrast, evaluation of the stream instance produces \textit{grounded} objects (i.e. with a concrete value). A stream instance can only be evaluated if its inputs are themselves grounded.

Existing PDDLStream algorithms expand the set of stream \textit{results} (objects and associated certified facts) in a breadth-first manner according to their so-called `level' . On a given iteration, all streams are instantiated up to a certain level $l$, and a PDDL planner is invoked on the resulting set.  If no plan is found, the next iteration starts with $l \leftarrow l+1$. The `level' of a stream is defined recursively as:
\begin{equation}
\label{eq:1}
\scriptsize
    \texttt{LEVEL}(s(\bar x)) = 1 + \textbf{count}(s(\bar x)) + \underset{ p \in s(\bar x).domain}{\max} \texttt{LEVEL}(p(\bar x).stream)
\end{equation}
Here, \textbf{count} refers to the number of times that a stream instance has been evaluated. A single stream instance may need to be evaluated many times to produce a value which results in a feasible trajectory. For example, the planner may consider many grasps for the same object before finding one that is collision-free.  

\section{Our Approach}
\label{sec:Approach}
Our approach uses a learned model to guide the expansion of optimistic facts in a PDDLStream problem. We first present a PDDLStream algorithm which uses a priority queue to order the instantiation of streams. Second, we formulate the problem of assigning a priority to optimistic stream instantiations as a learning problem. Finally, we propose a model architecture which makes use of the structure of streams to recursively build up representations of optimistic objects, and assign priorities to them simultaneously. 

\subsection{\texttt{INFORMED} Algorithm}
We propose the \texttt{INFORMED} algorithm, which uses a priority queue, $Q$, to order the addition of optimistic results to the planning problem based on a \textit{relevance score} given by a black-box model, $\mathcal M_\theta$. The algorithm is composed of four distinct phases: Expansion, Planning, Sampling, and Feedback. We describe each of these steps below, and provide the full pseudocode in \ref{alg:informed}.


\underline{\textit{Expansion}}. In each iteration of the algorithm, the stream result with the highest score is popped off of the queue and if the result is optimistic, its certified facts are added to $I_{opt}$. A sub-routine \texttt{EXPAND} creates all the child stream instances which use these newly added objects. These are all scored and pushed onto the queue for later expansion (line \ref{lst:expand} in Alg. 1).

\underline{\textit{Planning}}. In the planning phase, we invoke a PDDL planner on the current set of optimistic and grounded facts (line \ref{lst:plan} in Alg. 1). If an optimistic plan $\pi_{opt}$ is found, it is returned along with the \textit{stream plan} $\psi$, which is the sequence of stream instances that need to be evaluated to ground $\pi_{opt}$. This subroutine, which is a light wrapper around the underlying FastDownward planner, is the same as that presented in \cite{garrett2020pddlstream}. A key issue for our algorithm is how to decide when the set of included optimistic facts is sufficient for planning to succeed. This \texttt{SHOULD-PLAN} classifier may itself be a good candidate for learning from experience. However, in this work, we simply choose to plan every time $K = 100$ optimistic and grounded facts have been added to the problem, so in our implementation \texttt{SHOULD-PLAN} returns true every $K$ iterations.

\underline{\textit{Sampling}}. If a plan has just been found or if we have found a plan in a previous iteration without yet grounding all of its parameters, then we have the option to devote time to sampling. This means running the underlying generators for each of the streams in the plan $\psi$. This is facilitated by the \texttt{PROCESS-STREAMS} sub-routine  (line \ref{lst:sample}), of which several variants are described in \cite{garrett2020pddlstream}. In this work, we employ the version used by the \texttt{ADAPTIVE} algorithm in \cite{garrett2020pddlstream}. 

\underline{\textit{Feedback}}. If, in the course of sampling, the algorithm has managed to produce groundings for all the parameters in an optimistic plan, then our algorithm halts and the plan is returned (line \ref{lst:return_success}).
Otherwise, we would like to update our planning problem to account for information gleaned from sampling. By evaluating the streams that support an optimistic plan, \texttt{PROCESS-STREAMS} will produce grounded objects that we already suspect to be relevant. These new values are all immediately added to $I_{ground}$ so that they can be used by the planner in the next expansion and planning phases (line \ref{lst:add_new}). 

\begin{algorithm}[H]%
    \caption{\footnotesize{}INFORMED($\mathcal A, \mathcal S, \initialstate, \goalstate, \mathcal{M}_{\theta}$, PROCESS-STREAMS):}\label{alg:informed}
    \begin{algorithmic}[1]
    \footnotesize
    \State $I_{ground} = \initialstate$.copy(), $I_{opt} = \{\}$, $Q = []$
    \For{$s(\bar{x}) \in \text{INSTANTIATE}(\mathcal S, \initialstate)$} \label{lst:inst1}
        \State r = $s(\bar{x})$.next\_optimistic()
        \State score = $\mathcal{M}_{\theta}(r, I_{ground} \cup I_{opt}, \goalstate)$
        \State $Q$.push$(\langle score, r \rangle)$
    \EndFor
    \While{True}
        \If{$Q\text{.len()} > 0$}
            \State $\langle score, r \rangle = Q.pop()$
            \If{$r\text{.is\_optimistic()}$}
                \State $I_{opt}\text{.add}(r.certified)$
            \EndIf
            \For{$s(\bar{x}) \in \text{EXPAND}(r, I_{ground} \cup I_{opt}, \mathcal{S})$} \label{lst:expand}
                \State r = $s(\bar{x})$.next\_optimistic()
                \State score = $\mathcal{M}_{\theta}(r, I_{ground} \cup I_{opt}, \goalstate)$
                \State $Q$.push$(\langle score, r \rangle)$
            \EndFor
        \EndIf
        \State $\pi_{opt}$ = None
        \If{SHOULD-PLAN()}
            \State $\pi_{opt}, \psi= \text{OPTIMISTIC-SOLVE}(\mathcal{A}, I_{ground} \cup I_{opt}, \goalstate)$ \label{lst:plan}
        \EndIf
        \State $t = \text{SHOULD-SAMPLE}()$
        \If{$ t \le 0$ and $Q\text{.len}() = 0$}
            \State \textbf{return} None 
        \EndIf
        \If{$ t > 0$}
            \State $\begin{aligned}\pi, new, & processed = \\ &\text{PROCESS-STREAMS}(I_{ground} \cup I_{opt}, \psi , \pi_{opt}; t)\end{aligned}$\label{lst:sample}
            \If{$\pi \neq \text{None}$}
                \State \textbf{return}  $\pi$ \label{lst:return_success}
            \EndIf
            \For{$r \in new$}
                \State $I_{ground}.add(r.certified)$ \label{lst:add_new}
                \State score = $\mathcal{M}_{\theta}(r, I_{ground} \cup I_{opt}, \goalstate)$
                \State $Q\text{.push}(\langle score, r \rangle)$\label{lst:score_new}
            \EndFor
            \For{$r \in processed$}
                \State $I_{opt}\text{.remove}(r.certified)$\label{lst:remove_opt}
            \EndFor
        \EndIf
    \EndWhile
    \end{algorithmic}
\end{algorithm}
\vspace{-2em}
\begin{algorithm}[H]%
    \caption{\footnotesize{}INSTANTIATE ($\mathcal S, \mathcal I)$:}\label{alg:instantiate}
    \begin{algorithmic}[1]
    \footnotesize
    \State  $\{ s(\bar{x}) | \forall s \in \mathcal S, \forall p \in s.domain, |\bar{x}| = |s.input|, p(\bar{x}) \in \mathcal I\}$
    \end{algorithmic}
\end{algorithm}
\vspace{-2em}
\begin{algorithm}[H]
    \caption{\footnotesize{}EXPAND (r, $\mathcal S, \mathcal I)$:}\label{alg:expand}
    \begin{algorithmic}[1]
    \footnotesize
    \State $\begin{aligned}
            \{ s_j(\bar{x}) | \forall s_j(\bar{x}) \in \texttt{INSTANTIATE}(\mathcal S, r.certified \cup \mathcal I), r.output \subset \bar{x} \}
        \end{aligned}$
    \end{algorithmic}
\end{algorithm}
\vspace{-2em}




\subsection{Completeness of the \texttt{INFORMED} algorithm}
The algorithm described above inherits many of the same theoretical guarantees as those described in \cite{garrett2020pddlstream}. In particular, the proof of semi-completeness for optimistic algorithms (i.e. theorem 3) can be extended to the \texttt{INFORMED} algorithm under some assumptions about the relevance scores. We omit the proof from the paper for brevity, but include it in supplementary materials on the project webpage\footref{fnlabel}.

\subsection{Relevance of Streams as Binary Classification}
\label{sec:relevance}
The question of whether a stream instance is relevant to the planning problem can be boiled down to whether any of the objects that it produces are relevant. In this work we define an object to be relevant to a PDDLStream problem if: (1) it is part of the preimage of any optimal plan, or (2) it is used as input to a stream which produces a relevant object. We note, however, that applying this definition directly would be impractical. It is prohibitively expensive to compute optimal plans for problems of interest, so we apply this definition using satisficing plans found by existing PDDLStream algorithms.

\underline{\textit{Relevance of a PDDL object.}} In a TAMP setting, the objects in the preimage of a given plan will correspond to real valued quantities, such as the 7DOF configuration of a robotic arm. However, these values are rarely ever unique; that a specific grasp configuration of an object was used in a plan should imply that a different grasp of the same object is also important to consider. Therefore, when deciding on the relevance of an object, we concern ourselves not with the actual numeric value of the object, but with the set of constraints that act on it. 
To determine whether an object is relevant, we simply need to check whether there exists a counterpart in a given preimage which has the same set of constraints (i.e. the same ancestor streams).

\underline{\textit{Training data.}} This approach allows us to cast the relevance of stream instances as a supervised binary classification problem. Given a set of training problems $D = \{\langle \mathcal O_i, \initialstate_i, \goalstate_i \rangle\}_{i = 1}^N$ defined by objects $\mathcal O_i$, initial state $\initialstate_i$, and goal $\goalstate_i$ we use an existing PDDLStream solver to find a plan $\pi_i$, and record all the stream instances $s_{ij}$ produced along the way. For each $s_{ij}$, we use $\text{PREIMAGE}(\pi_i)$ to assign a binary label $y_{ij}$ as described above. 

\subsection{Modelling Relevance of Streams}
\label{sec:Scoring}

Given the $\initialstate$, $\goalstate$, the current set of optimistic stream results $I_{opt}$ and grounded stream results $I_{ground}$, we want to assign a \textit{relevance score} $y \in [0, 1]$ to a stream result $s(\bar{x}) \to \bar{o}$.

\underline{\textit{MLP representations of streams.}} Our model architecture is motivated by the observation that a stream represents constraint(s) between its inputs $\bar{x}$ and outputs $\bar{o}$, and that for each domain there are only a handful of explicitly defined streams. We model each stream $s$ with a simple multi-layer perceptron (MLP) $M_{s}$, which has fixed input and output sizes proportional to $|s.\text{input}|$ and $|s.\text{output}|$. Given fixed dimensional ($N=64$ in our experiments) dense embeddings $h_{x_i}$ of the input objects $\bar x$, $M_{s}$ produces an embedding $h_{o_j}$ of the same dimension to each of the output objects in $\bar o$, as well as a scalar score $y$ representing the relevance of those outputs to the planning problem. For an instance of stream $s$ with input objects $\bar x$, the computation of $M_{s}$ is as follows (note $\odot$ denotes concatenation):
\begin{equation}
\label{eq:encode_mlp}
h = M_{s}^{encoder}(\underset{x_i \in \bar x}{\odot} h_{x_i})
\end{equation}
\begin{equation}
\label{eq:score_mlp}
y = M_{s}^{scorer}(h)
\end{equation}
\begin{equation}
\label{eq:decode_mlp}
h_{o_j} = M_{s_i}^{decoder^j}(h)
\end{equation}

As depicted in the bottom right of Fig.~\ref{fig:model}, these $M_{s}$ can be recursively combined in exactly the same way as streams, producing relevance scores and dense vector representations for each object their associated streams create. A similar architecture has been used for natural language processing \cite{socher2011parsing} where it is referred to as a tree-recursive recurrent neural network. \textbf{Example.} Figure \ref{fig:model} depicts the flow of computation in our model for a simple example. In order to predict a score for the grasp pose $g_0^*$ of block $b_2$, the GNN embedding of $b_2$ is passed into the MLP for the grasp stream, $M_{grasp}$. The score is then computed as in equations \ref{eq:encode_mlp}, \ref{eq:score_mlp}. Similarly, the vector representation of $g_0^*$ is computed as in equation \ref{eq:decode_mlp} to produce $h_{g_0^*}$, which can then be used in subsequent streams.
\begin{figure}
    \centering
    \includegraphics[width = 0.9\linewidth]{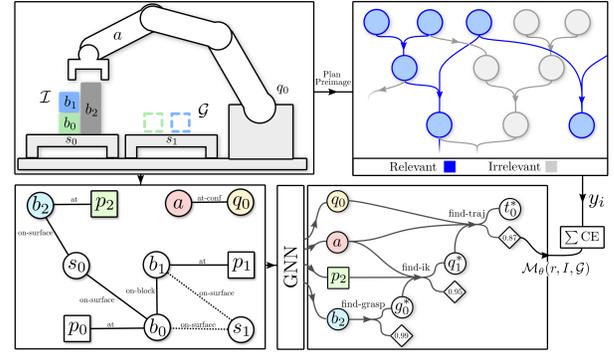}
    \caption{Top row: A block stacking task is solved and a set of labeled stream results is extracted from the plan. Bottom left: the Problem Graph representation of the problem. Bottom right: GNN-Stream is trained to predict relevance of optimistic stream results.}
    \label{fig:model}
\end{figure}

\underline{\textit{GNN representations of objects and facts.}} To obtain the representations for the initial objects in the planning problem, we employ \textit{Graph Neural Networks} owing to their ability to represent objects based on their relations to one another. We provide details about the architecture and hyperparameters of the GNN in supplementary materials. The input to the GNN, which we call the \textit{Problem Graph} (PG), is a graph with nodes representing the PDDL objects and edges representing the facts that relate those objects to one another in both $\initialstate$ and $\goalstate$ (see bottom left of Fig.~\ref{fig:model}). This relational representation is similar to that used in \cite{silver2020planning}, and allows our system to generalize to different numbers and identities of objects. The node feature vectors consist of the 3D position of the object if it exists, and padding otherwise. The edge feature vectors have a one-hot encoding of their corresponding predicate, and a flag if the fact appears in $\initialstate$ (1) or $\goalstate$ (0). This graph representation is extensible and can include other geometric (or semantic) information about the scene. Since our experiments use only cubes of fixed sizes, we do not need to encode the geometry of each object in the scene. However, in principle, data such as point-clouds or depth-images can be used as node features to support planning in more diverse scenarios, but this is out of scope for this work.

The end-to-end model architecture $\mathcal M_\theta$ is depicted in the bottom row of Fig.~\ref{fig:model}. A given planning problem is used to produce a Problem Graph. A GNN produces an embedding of each object according to its relations in both the goal and the initial state. These representations can then be used as input to stream MLPs, which are combined recursively to obtain representations and scores for optimistic objects produced by their associated streams. This entire computation is differentiable, so $\mathcal{M}_{\theta}$ is trained end-to-end to minimize a weighted cross-entropy, giving a higher penalty to false negatives. $\mathcal L(D, \theta) = \sum_{i = 1}^N \sum_{j = 1}^{M_i} \mathcal{L}_{CE}(y_{ij}, \mathcal{M}_{\theta}(\initialstate_i, \goalstate_i, s_{ij}(\bar{x}_{ij})))$

\section{Experimental Results}
\label{sec:Experiments}
\subsection{Environment Setup.}
Evaluation of \texttt{INFORMED} was conducted across six problem types involving a 7DoF robot arm, loosely based on those outlined in \cite{lagriffoul2018platform}. All experiments are conducted using 2 cores of an Intel Broadwell processor with an 8GB memory limit. For each problem type, we randomly generate 100 instances which are used to train a separate model and 100 held-out test instances which we use for evaluation. We additionally report results for training a single `multi-task' model in table \ref{tab:combined}. \textit{The training instances are drawn from distributions with few objects/goals so that the baseline is able to solve the majority of them within the timeout. We sample the test instances from a more challenging distribution to evaluate the model's ability to generalize to problems with more objects and goal conjuncts than those seen during training.} Initial placement of objects and their goal states are randomly generated in each problem, so that even problems of the same type with the same number of objects may have different solutions. All methods use a 90 second planning timeout, so we report the proportion of problems solved within the timeout, as well as average planning times for solved instances.

\begin{figure}
     \centering
     \caption{Typical example of the evaluation tasks in each environment.}
     \begin{subfigure}[b]{0.32\linewidth}
         \centering
         \includegraphics[width=\textwidth]{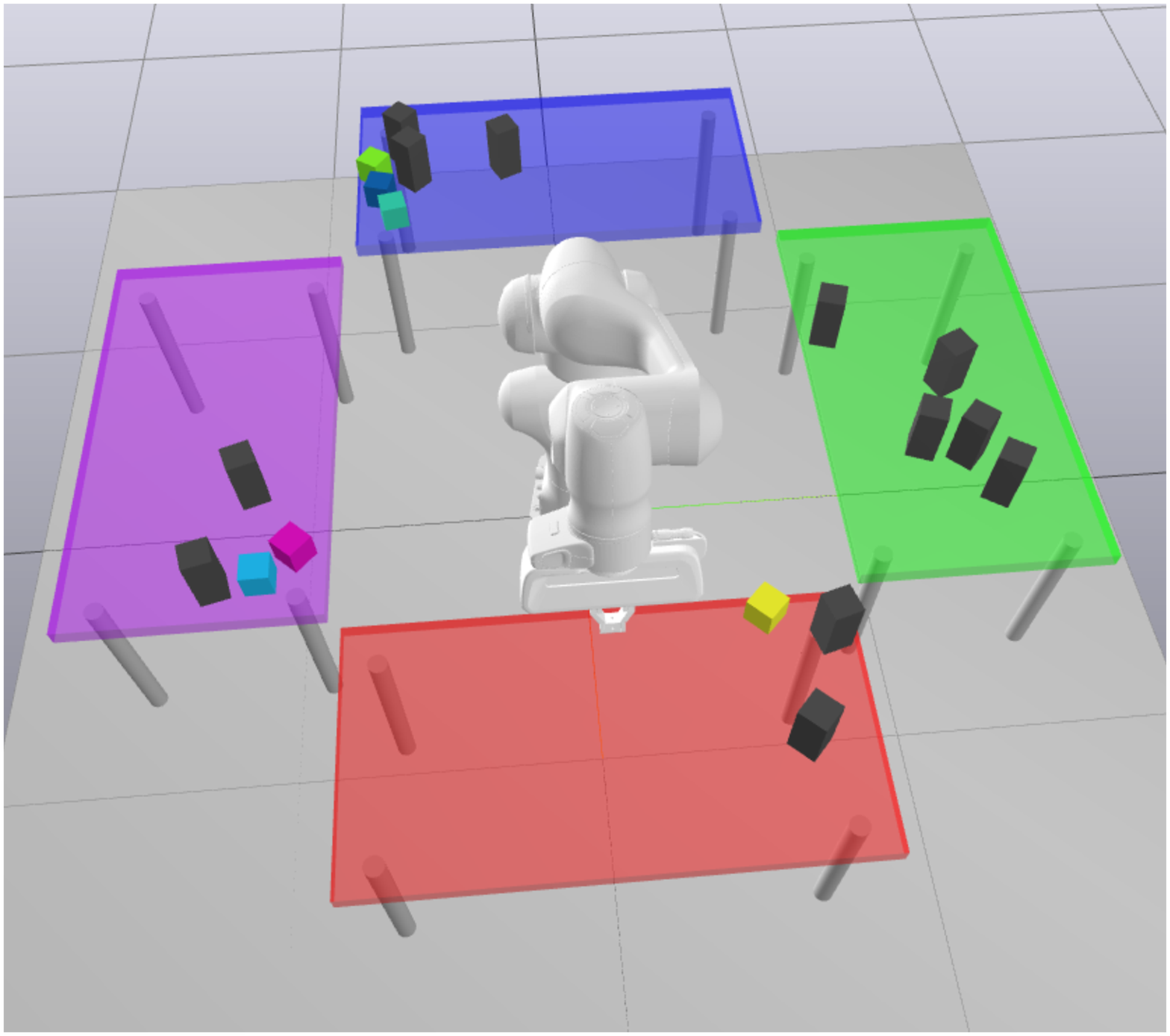}
         \caption{Clutter}
         \label{fig:clutter}
     \end{subfigure}
     \begin{subfigure}[b]{0.32\linewidth}
         \centering
         \includegraphics[width=\textwidth]{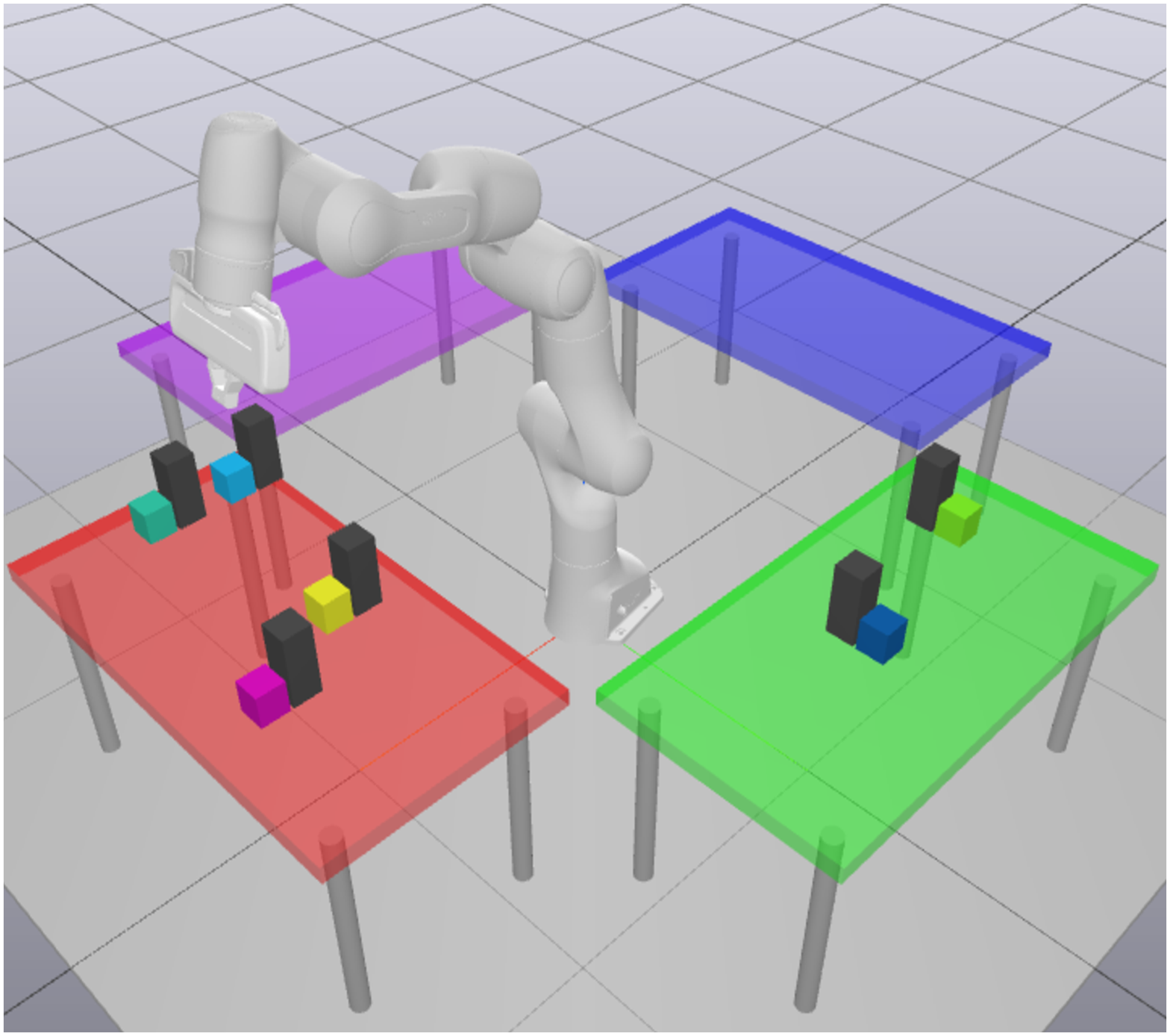}
         \caption{Non-monotonic}
         \label{fig:non-monotonic}
     \end{subfigure}
     \begin{subfigure}[b]{0.32\linewidth}
         \centering
         \includegraphics[width=\textwidth]{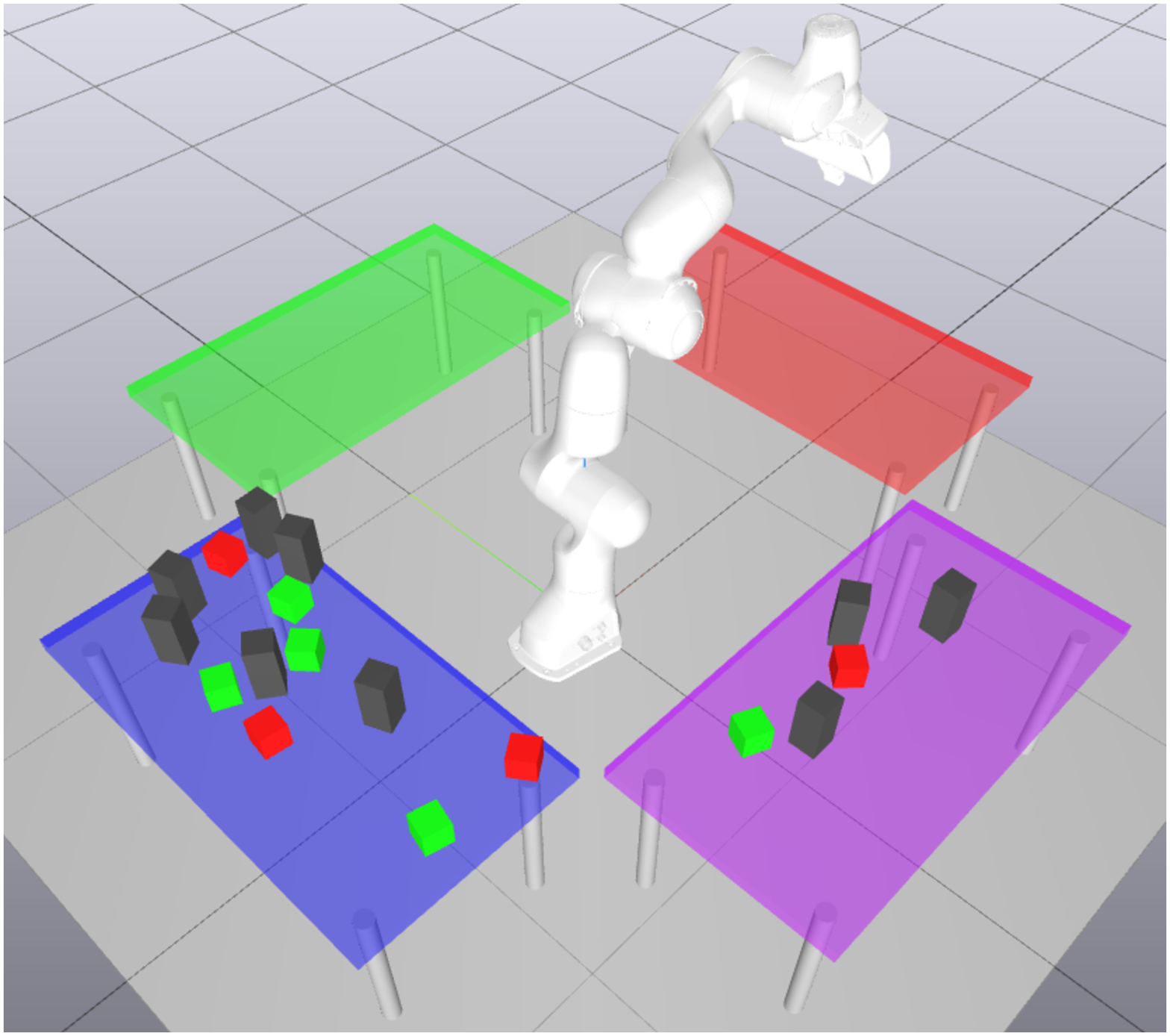}
         \caption{Sorting}
         \label{fig:sorting}
     \end{subfigure}
     \begin{subfigure}[b]{0.32\linewidth}
         \centering
         \includegraphics[width=\textwidth]{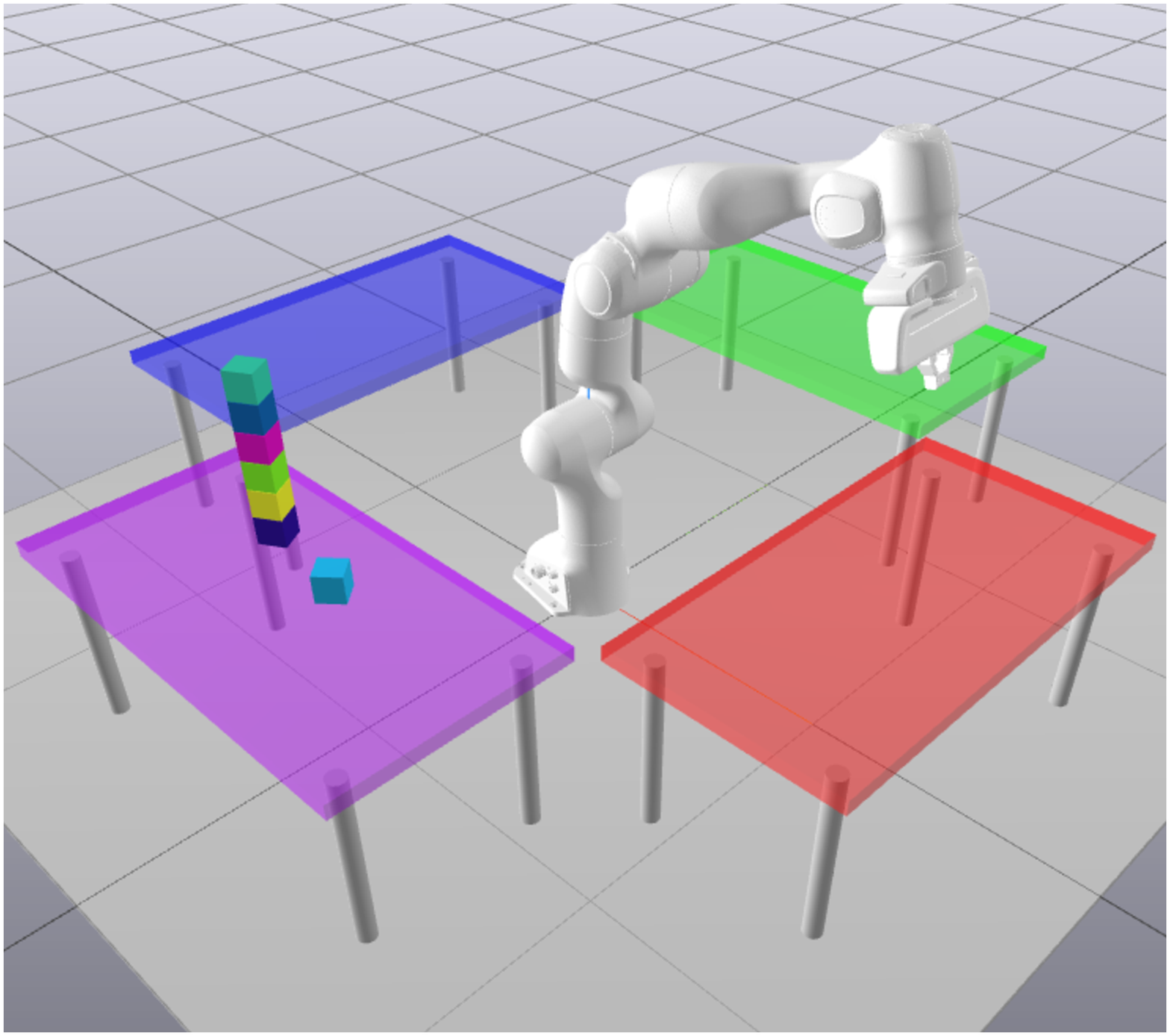}
         \caption{Stacking}
         \label{fig:stacking}

     \end{subfigure}
     \begin{subfigure}[b]{0.32\linewidth}
         \centering
         \includegraphics[width=\textwidth]{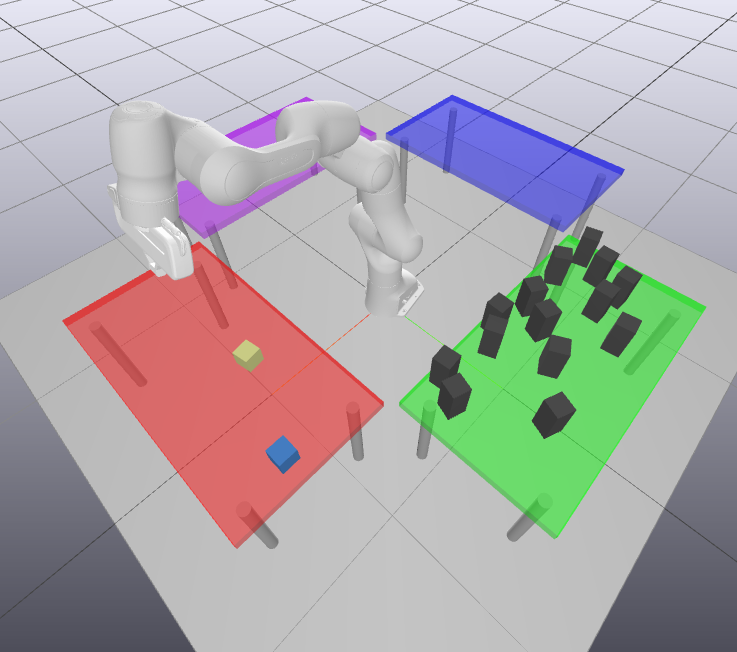}
         \caption{Distractors}
         \label{fig:distractors}

     \end{subfigure}
     \begin{subfigure}[b]{0.32\linewidth}
         \centering
         \includegraphics[width=\textwidth]{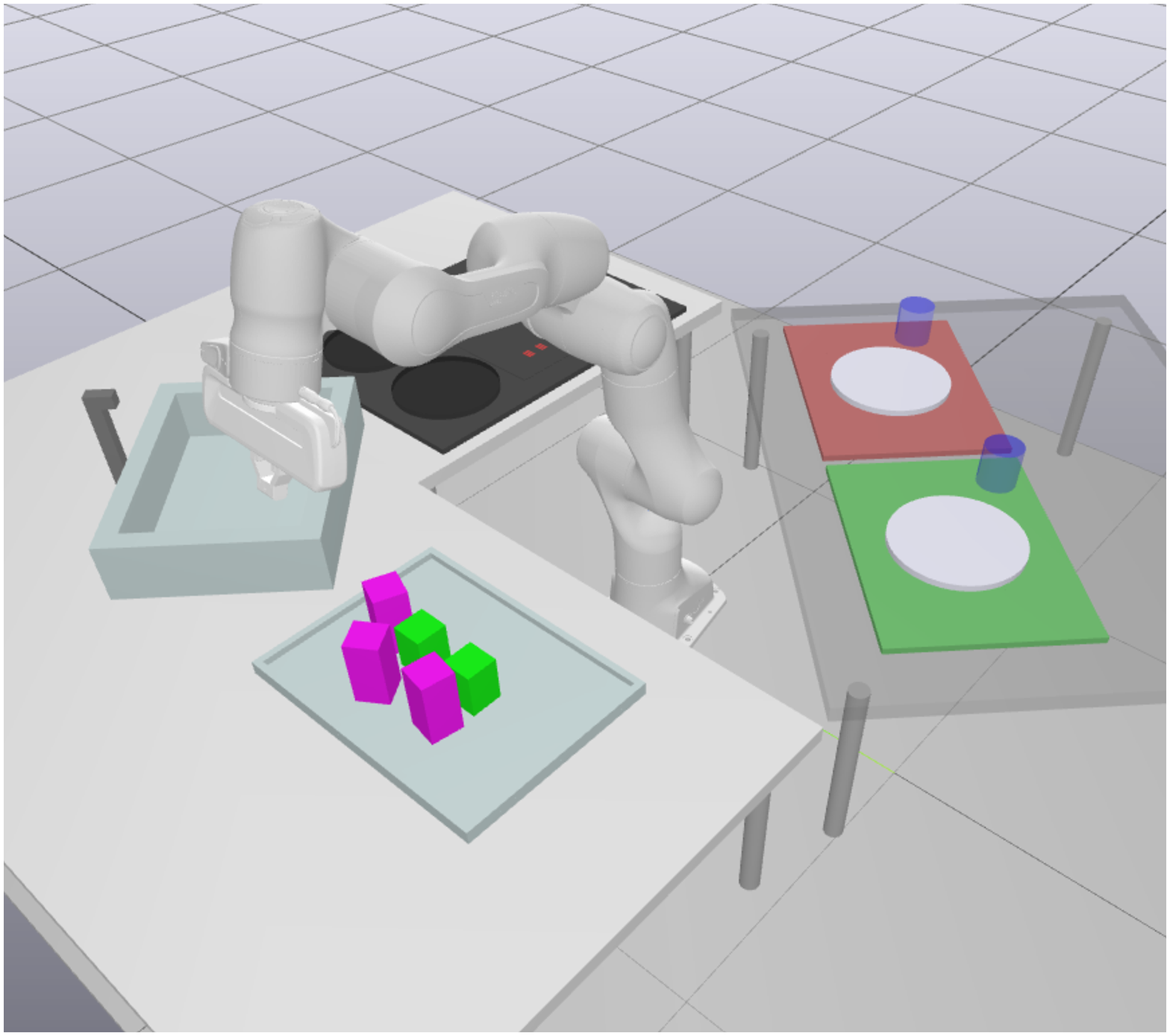}
         \caption{Kitchen}
         \label{fig:kitchen}
     \end{subfigure}
     \label{fig:experiments}
\end{figure}

\textbf{Problem Types:} Blocks world problems are divided into five categories which share a domain definition, but present different challenges to the planner, as shown in Fig.~\ref{fig:experiments}.  There are two types of blocks in these problems (not distinguished logically): blocks (shorter) and blockers (taller) (Figs.~\ref{fig:clutter} and ~\ref{fig:distractors}). The robot arm is surrounded by four tables.

\underline{Stacking:} Blocks are arranged randomly in each scene, and the goal is to assemble them into specific towers. These problems involve long plans, low likelihood of obstructions. Training problems involved up to 4 objects and test problems up to 7.

\underline{Sorting:} Colorful are arranged arranged randomly, and the goal is to move them to the table with the corresponding color. Blockers may need to be moved if they obstruct a plan, but must be returned to their original tables. Sorting problems have high branching factor, long plans, high likelihood of obstructions. Training problems involved up to 14 objects and test problems up to 20.

\underline{Clutter:} Blocks need to be stacked or rearranged in the presence of twice as many blockers which may obstruct motions, grasps or placements. Initial positions are sampled using ordered Poisson-Disc Sampling so that blocks are placed close together. These problems have a high branching factor, short plans, high likelihood of obstructions. Training problems involved up to 15 objects, test problems up to 18.

\underline{Non Monotonic:} Blockers and blocks come in pairs (see Fig.~\ref{fig:non-monotonic}). Blocks have to be moved to a different table, whereas blockers must be returned to their exact starting positions. It is always infeasible to grasp the block without first moving its corresponding blocker, so the planner must temporarily undo subgoals (hence the name). These problems have low branching factor, long plans, guaranteed obstructions. Training problems involved up to 6 objects and test problems up to 8.

\underline{Distractors:} Blocks need to be stacked or rearranged in the presence of "distractor" blockers which are placed on their own table. Unlike the blockers in other problem types, distractors do not appear in the goal, and do not need to be interacted with. This problem type tests the planner's ability to ignore irrelevant objects. Test problems contain 2-3 blocks and up to 50 distractors. There are no training problems - instead, the model trained on the Clutter task is used here.

\underline{Kitchen:} In these problems, a robot arm must complete tasks in a kitchen involving different types of items. Items can be cleaned if they are in the sink, and clean items can be cooked if they are on a burner (Fig.~\ref{fig:kitchen}). The plans in this domain can be long, but the branching factor is lower than blocks world due to the absence of stacking actions. 

\textbf{Methods:} We compare our approach to \texttt{ADAPTIVE}, the best-performing PDDLStream algorithm presented in \cite{garrett2020pddlstream}. Our approach differs from \texttt{ADAPTIVE} in two key respects. First, while \texttt{ADAPTIVE} must expand all streams up to a given depth before planning, \texttt{INFORMED} will invoke the planner after a constant number of additional streams have been expanded. The second key difference is that \texttt{INFORMED} uses a given relevance model to order the expansion of streams. In order to isolate the effect of the model from that of planning frequency, we evaluate the following variants of our main method. Performance differences between these ablations and \texttt{INFORMED} can only be attributed to the learned model.
\begin{itemize}[wide, labelwidth=!, labelindent=0pt,noitemsep,nolistsep]
    \item \texttt{INFORMED} (Level):  Stream expansion is ordered by level (as defined in equation \ref{eq:1}). Ties are broken arbitrarily. This ablation does not use the training data. This ordering is the same as that used by \texttt{ADAPTIVE}, so only differs in the frequency with which the PDDL planner is invoked.
    \item \texttt{INFORMED} (Stats):  Stream expansion is ordered based on the proportion of times the associated stream was marked relevant in the training data. This ablation assigns the same score to all instances of a given stream.
    \item \texttt{INFORMED} (PG): An ablation of our model architecture which makes use of the problem graph network, but not the recursively combined stream MLPS. The problem graph network assigns an importance score to each object in the initial scene. During planning, a candidate stream instance's priority is the average importance score of its ancestors, divided by its depth.
    \item \texttt{INFORMED} (Ours): Our proposed approach, as described in section \ref{sec:Scoring}.
\end{itemize}
Finally, we also compare to a prior work \cite{silver2020planning} (\texttt{PLOI}) which uses learned GNNs to predict the importance of objects in the planning problem. \texttt{PLOI} treats the PDDLStream solver as a black box, and simply excludes all objects below a threshold score from the call to the solver. In order to maintain completeness, an outer-loop iteratively reduces the importance threshold until a plan is found, or all objects pass the threshold.

\textbf{Refined vs. unrefined mode:} \texttt{ADAPTIVE} and \texttt{INFORMED} admit two modes of operation with respect to stream expansions. The first, \textit{refined} expansion, operates exactly as described in \cite{garrett2020pddlstream} and section \ref{sec:Setup} - optimistic objects from different stream instances are assigned unique identifiers which are used in further stream expansions. In contrast, the second mode of operation, \textit{unrefined}, groups together all objects produced by a specific stream under a single shared identifier. This mode has the effect of bounding the combinatorial explosion from stream expansions, since it effectively limits the number of new objects. This comes at the expense of additional work, which must be done to disambiguate shared identifiers in the event that a plan is found which uses them. In our testing, we observed that this trade-off often results in faster planning times for \texttt{ADAPTIVE}. The relevance model we have described cannot score these unrefined streams, because of the inherent ambiguity in their meaning. In order to evaluate the effect of the model, we use both algorithms in \textit{refined} mode. However, for the sake of completeness, we also report \texttt{ADAPTIVE}'s performance in unrefined mode. 

\subsection{Results and Discussion}

\begin{table}[h!]
\caption{Planning performance on held-out \textit{test} problems. All runs are conducted with 90s timeouts. We report the number of solved instances (out of 100), and the average times over all \textbf{solved} problems of a given type.}
\tabcolsep=0.05cm
\centering
\tiny
\begin{tabular}{c|cc|cc|cc|cc|cc|cc|cc}
\toprule
 & \multicolumn{2}{c|}{\textbf{\texttt{ADAPTIVE}}} & \multicolumn{2}{c|}{\textbf{\texttt{ADAPTIVE}}} & \multicolumn{2}{c|}{\textbf{\texttt{INFORMED}}}& \multicolumn{2}{c|}{\textbf{\texttt{INFORMED}}} & \multicolumn{2}{c|}{\textbf{\texttt{INFORMED}}}  & \multicolumn{2}{c|}{\textbf{\texttt{INFORMED}}}  & \multicolumn{2}{c}{\textbf{\texttt{\texttt{PLOI}}}} \\
 & \multicolumn{2}{c|}{\textbf{(Unrefined)}} & \multicolumn{2}{c|}{\textbf{(Refined)}} & \multicolumn{2}{c|}{\textbf{(Level)}}& \multicolumn{2}{c|}{\textbf{ (Stats)}} & \multicolumn{2}{c|}{\textbf{ (PG)}} & \multicolumn{2}{c|}{\textbf{ (Ours)}} & \multicolumn{2}{c}{\cite{silver2020planning}} \\
\midrule
\textbf{Domain} & \textbf{Solved} & \textbf{Time}  & \textbf{Solved} & \textbf{Time} & \textbf{Solved} & \textbf{Time}& \textbf{Solved} & \textbf{Time} & \textbf{Solved} & \textbf{Time}& \textbf{Solved} & \textbf{Time} & \textbf{Solved} & \textbf{Time} \\  \midrule
\textbf{Clutter      } & \textbf{54} &30.99&           22 &    18.20 &  5 &    37.55 & 25 &   11.67 & 32 & 15.22&       \textbf{54} &    12.15  & 49 &  16.11 \\
\textbf{Non Mono} &  27&31.45&         28 &    36.96 &  15 &    73.99 & 0 &    - &      1 & 35.68 & \textbf{58} &    33.88    & 25 & 29.71\\
\textbf{Sorting      } & 66&16.95&          50 &    25.11 &  12 &    15.46 & 62 &    17.65 & 51&  14.60&   \textbf{77} &    21.07   & 68 & 20.47\\
\textbf{Stacking     } &   47&16.95&        20 &    5.52 &  21 &   12.38 & 16 &    4.70 & 21 & 9.04 &    \textbf{54} &    7.94  & 47 & 9.20\\
\textbf{Distractors  } &  77 &47.82&         1 &    42.74 &  1 &   23.17 & 0 &   - &  66 & 23.21 &   \textbf{100} &    9.65   & \textbf{100}  & 6.62\\
\textbf{Kitchen      } & 91 & 19.90 & \textbf{95} & 27.24 & 83 & 24.51 & 84 & 28.52 & 86 & 24.51 &94 & 23.93   & 91 &28.76 \\
\bottomrule
\end{tabular}
\label{tab:refined}
\end{table}
\begin{figure}[h!]
    \centering
    \input{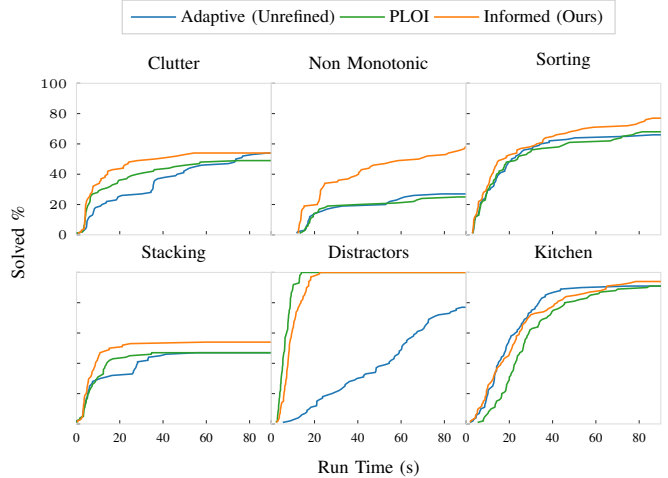}
     \vspace{-1.6em}
    \caption{A comparison of \texttt{ADAPTIVE}, \texttt{PLOI} and \texttt{INFORMED} on percentage of test-distribution problems solved as a function of run time. Test-distribution includes problems with more objects and goals than seen in training.}
    \label{fig:performance_charts}
 \end{figure}
 
\underline{\textit{Ablations.}} Table \ref{tab:refined} shows the success rates and planning times for the ablations described above. On all problem types, the learned model is able to significantly outperform ablations, showing that it is able to prioritize the inclusion of task-relevant objects enabling the planner to solve more problems within the timeout. The Level ablation generally performs poorly compared to \texttt{ADAPTIVE}, despite using the same underlying ordering of stream expansions. This is primarily due to more frequent planner invocations using up more of the allotted time. We found that the Stats ablation is able to significantly outperform \texttt{ADAPTIVE} in Sorting problems, when using refined expansion. Upon closer inspection, we found that the training examples for Sorting contain no stack or unstack actions, which are indeed unnecessary in this task. By assigning a score of zero to objects which would be required for those actions, the Stats ablation is able to effectively reduce the branching factor of its search, and thus solve more problems compared to \texttt{ADAPTIVE}. The PG ablation performed significantly better than others on Distractors problems, where many initial objects were wholly irrelevant. However, it does not make use of the precise structure of relations that relate the stream's outputs to its ancestor objects, and is thus unable to match the performance of our proposed approach in any of the problem types. This demonstrates the effectiveness of the recursive stream MLP's at making use of the structure underlying stream instances. 

\underline{\textit{Comparison to \texttt{PLOI}.}} \texttt{PLOI} does quite well on distractors, but performs on par with \texttt{ADAPTIVE} on all other problem types. This is because apart from the Distractors problems, the majority of objects in the scene are needed to find a plan. Therefore, \texttt{PLOI} does not provide a significant advantage in planning time by excluding low-scored objects, and may incur overhead by having to invoke the PDDLStream planner several times until all the necessary objects have been included. In contrast, our approach is able to retain the benefit of \texttt{PLOI} when irrelevant objects are present (as evidenced by the Distractors experiments) as well as provide speed-up in other problem types by making finer-grain predictions about the relevance of stream instances.

\begin{wrapfigure}{R}{4cm}
\centering
\begin{tikzpicture}

\definecolor{color0}{rgb}{0.12156862745098,0.466666666666667,0.705882352941177}
\definecolor{color1}{rgb}{1,0.498039215686275,0.0549019607843137}

\begin{axis}[
legend cell align={left},
legend style={
  fill opacity=0.8,
  draw opacity=1,
  text opacity=1,
  at={(-0.08,1.54)},
  anchor=north west,
  draw=white!80!black,
  nodes={scale=0.75, transform shape},
  font=\scriptsize
},
tick align=outside,
tick pos=left,
x grid style={white!69.0196078431373!black},
xlabel={Num Distractors},
label style={font=\scriptsize},
ticklabel style={font=\tiny},
xmin=8, xmax=52,
xtick style={color=black},
y grid style={white!69.0196078431373!black},
ylabel={Run Time (s)},
ymin=-1.23968414068222, ymax=95,
ytick style={color=black},
width=4cm
]
\addplot [semithick, color0]
table {%
10 27.836953719457
12 42.9660796324412
14 8.80210781097412
16 25.5253311793009
18 28.9076856613159
20 44.5588691830635
22 32.4760802057054
24 45.997216129303
26 55.945739712034
28 48.7438741922379
30 53.4140679836273
32 49.0010636647542
34 65.9398797988892
36 89.3690360387166
38 90
40 64.3701205253601
42 79.43411260181
44 74.958025654157
46 76.6812614713396
48 80.9999187588692
50 81.4899971485138
};
\addlegendentry{Adaptive (Unrefined)}
\addplot [semithick, color1]
table {%
10 3.76093157132467
12 6.36953401565552
14 3.10506272315979
16 5.41705663998922
18 5.87399892807007
20 6.4616858959198
22 6.8325252532959
24 7.18138761520386
26 8.3428635937827
28 7.92933714389801
30 9.49386048316956
32 7.34322500228882
34 8.7594388961792
36 13.1680072148641
38 11.8101433515549
40 15.5381011565526
42 13.5230958991581
44 12.0839349428813
46 13.2190117835999
48 12.8061344623566
50 13.0825231671333
};
\addlegendentry{Informed (Ours)}
\end{axis}

\end{tikzpicture}
\caption{Average planning time as a function of the number of distractors.} 
\label{fig:distractors-time}   
\end{wrapfigure}
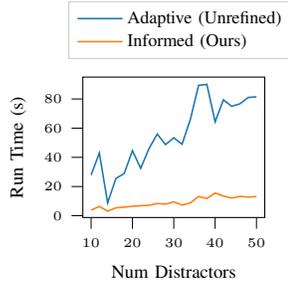
\underline{\textit{Generalization.}} All of the problems we use for evaluation are from a held-out test set. In addition, we generate these problems from test distributions with more objects, and more goal conjuncts, so the problem instances are unseen at training time. In general, the increased performance shown in Table \ref{tab:refined} demonstrates generalization to scenes with different numbers and placements of objects as well as different goals. The Distractors task is especially notable since this kind of problem is not part of any training set, and contained many more objects than any training problem. This demonstrates the ability of our architecture to generalize to problems of different size and composition.

The Distractors experiment also demonstrates a key advantage of our approach when compared to \texttt{ADAPTIVE}. Figure \ref{fig:distractors-time} shows the planning time of both methods as a function of the number of distractor objects. Even though the problems are no more geometrically or logically challenging, the simple inclusion of additional objects increases the planning time rapidly. In contrast, our method is able to maintain a comparatively constant planning time. 

\begin{wraptable}{r}{4cm}
\caption{Percentage of Stacking problems solved across varying stack heights. }
\tabcolsep=0.05cm
\scriptsize
\begin{tabular}{ccc}
\toprule
 & \textbf{\texttt{ADAPTIVE}} & \textbf{\texttt{INFORMED}} \\
 \textbf{Stack Height}   & \textbf{(Unrefined)} & \textbf{(Ours)} \\ \midrule
\textbf{2}              &    100.00 &      100.00 \\
\textbf{3}              &    100.00 &      100.00 \\
\textbf{4}              &     93.33 &      100.00 \\
\textbf{5}              &      0.00 &       72.22 \\
\textbf{6}              &      0.00 &       57.14 \\
\bottomrule
\end{tabular}

\label{tab:stacking_height}
\end{wraptable}As shown in Table \ref{tab:stacking_height}, \texttt{ADAPTIVE} is unable to solve any stacking problems with stacks of 5 or 6 blocks. \texttt{INFORMED}'s performance also deteriorates with stack height, however it is able to solve the majority of these problems despite the training set having only problems with up to 4 blocks. We found this improvement to be primarily due to a reduction in time spent solving the induced planning problems compared to \texttt{ADAPTIVE}. Note that no problems with a stack higher than 4 were seen during training, therefore our approach shows out-of-distribution generalization to problems with stacks composed of more objects. 

 \underline{\textit{Comparison of Planning Time Breakdown.}} We compared the breakdown of run time between \texttt{INFORMED} and \texttt{ADAPTIVE} in order to shed light on the differences in how the allotted time is spent. In figure \ref{fig:runtime_breakdown} we report these average breakdowns for stacking and clutter problems. We observed that the time spent solving induced planning problems is drastically decreased by the use of the learned model. In particular, search time for \texttt{ADAPTIVE} increases quite sharply as a function of the stack height in stacking problems, and the model is able to effectively address this issue. We also observed that the time spent performing model inference remains relatively low, though this also increases as a function of the number of objects.

\begin{figure}[h!]
     \centering
     \begin{tikzpicture}

\definecolor{color0}{rgb}{0.12156862745098,0.466666666666667,0.705882352941177}
\definecolor{color1}{rgb}{1,0.498039215686275,0.0549019607843137}
\definecolor{color2}{rgb}{0.172549019607843,0.627450980392157,0.172549019607843}
\def\width{0.55*\linewidth}
\pgfplotsset{every tick label/.append style={font=\tiny}}

\begin{groupplot}[
    group style={
        group size=2 by 1,
        horizontal sep = 1 cm,
        vertical sep = 0.0 cm
    }
]

\nextgroupplot[
legend cell align={center},
legend columns = 5,
legend style={fill opacity=0.8, draw opacity=1, text opacity=1, draw=white!80!black, at={(0.02,1.3)}, anchor=north west},
width = \width,
axis line style={white!86.6666666666667!black},
legend cell align={left},
tick align=outside,
x grid style={white!80!black},
xlabel={Stack Height},
xmajorticks=true,
xmin=1.415, xmax=4.585,
xtick style={color=white!15!black},
y grid style={white!80!black},
ylabel={Runtime (s)},
ymajorticks=true,
ymin=0, ymax=40,
ytick style={color=white!15!black},
tick pos=left
]
\addlegendimage{ybar,ybar legend,draw=none,fill=color0, pattern = north east lines}
\addlegendentry{Informed}
\addlegendimage{ybar,ybar legend,draw=none,fill=white}
\addlegendentry{Adaptive}
\draw[draw=white,fill=color0,] (axis cs:1.65,0) rectangle (axis cs:2,2.2995737195015);
\addlegendimage{ybar,ybar legend,draw=none,fill=color0}
\addlegendentry{Sampling}
\addlegendimage{ybar,ybar legend,draw=none,fill=color1}
\addlegendentry{Search}
\addlegendimage{ybar,ybar legend,draw=none,fill=color2}
\addlegendentry{Inference}
\draw[draw=white,fill=color0] (axis cs:1.65,0) rectangle (axis cs:2,2.80628267248294);
\draw[draw=white,fill=color0] (axis cs:2.65,0) rectangle (axis cs:3,5.72302770034115);
\draw[draw=white,fill=color0] (axis cs:3.65,0) rectangle (axis cs:4,7.825153438368);
\draw[draw=white,fill=color0,postaction={pattern=north east lines, pattern color=white}] (axis cs:2,0) rectangle (axis cs:2.35,2.49706534840328);
\draw[draw=white,fill=color0,postaction={pattern=north east lines, pattern color=white}] (axis cs:3,0) rectangle (axis cs:3.35,5.02397804793221);
\draw[draw=white,fill=color0,postaction={pattern=north east lines, pattern color=white}] (axis cs:4,0) rectangle (axis cs:4.35,6.96479746379334);
\draw[draw=white,fill=color1] (axis cs:1.65,2.80628267248294) rectangle (axis cs:2,3.14362835884094);
\draw[draw=white,fill=color1] (axis cs:2.65,5.72302770034115) rectangle (axis cs:3,7.28663231432438);
\draw[draw=white,fill=color1] (axis cs:3.65,7.825153438368) rectangle (axis cs:4,31.9603623072306);
\draw[draw=white,fill=color1,postaction={pattern=north east lines, pattern color=white}] (axis cs:2,2.49706534840328) rectangle (axis cs:2.35,2.89246629079183);
\draw[draw=white,fill=color1,postaction={pattern=north east lines, pattern color=white}] (axis cs:3,5.02397804793221) rectangle (axis cs:3.35,5.39037927985191);
\draw[draw=white,fill=color1,postaction={pattern=north east lines, pattern color=white}] (axis cs:4,6.96479746379334) rectangle (axis cs:4.35,9.19351584116618);
\draw[draw=white,fill=color2] (axis cs:1.65,3.14362835884094) rectangle (axis cs:2,3.14362835884094);
\draw[draw=white,fill=color2] (axis cs:2.65,7.28663231432438) rectangle (axis cs:3,7.28663231432438);
\draw[draw=white,fill=color2] (axis cs:3.65,31.9603623072306) rectangle (axis cs:4,31.9603623072306);
\draw[draw=white,fill=color2,postaction={pattern=north east lines, pattern color=white}] (axis cs:2,2.89246629079183) rectangle (axis cs:2.35,3.21067458788554);
\draw[draw=white,fill=color2,postaction={pattern=north east lines, pattern color=white}] (axis cs:3,5.39037927985191) rectangle (axis cs:3.35,5.66765414178371);
\draw[draw=white,fill=color2,postaction={pattern=north east lines, pattern color=white}] (axis cs:4,9.19351584116618) rectangle (axis cs:4.35,10.3032434622447);

\nextgroupplot[
width = \width,
axis line style={white!86.6666666666667!black},
tick align=outside,
x grid style={white!80!black},
xlabel={Num Blocks},
xmajorticks=true,
xmin=1.415, xmax=6.585,
xtick style={color=white!15!black},
ylabel style = {at = {(-0.11, 0.5)}},
y grid style={white!80!black},
ymajorticks=true,
ymin=0, ymax=60,
ytick style={color=white!15!black},
tick pos=left
]
\draw[draw=white,fill=color0] (axis cs:1.65,0) rectangle (axis cs:2,3.94728505352314);
\draw[draw=white,fill=color0] (axis cs:2.65,0) rectangle (axis cs:3,6.36479051401184);
\draw[draw=white,fill=color0] (axis cs:3.65,0) rectangle (axis cs:4,15.8168080606781);
\draw[draw=white,fill=color0] (axis cs:4.65,0) rectangle (axis cs:5,16.0077846008701);
\draw[draw=white,fill=color0] (axis cs:5.65,0) rectangle (axis cs:6,16.0836326358871);
\draw[draw=white,fill=color0,postaction={pattern=north east lines, pattern color=white}] (axis cs:2,0) rectangle (axis cs:2.35,3.25712248885516);
\draw[draw=white,fill=color0,postaction={pattern=north east lines, pattern color=white}] (axis cs:3,0) rectangle (axis cs:3.35,7.35013741120685);
\draw[draw=white,fill=color0,postaction={pattern=north east lines, pattern color=white}] (axis cs:4,0) rectangle (axis cs:4.35,10.8525044935771);
\draw[draw=white,fill=color0,postaction={pattern=north east lines, pattern color=white}] (axis cs:5,0) rectangle (axis cs:5.35,14.5752984466611);
\draw[draw=white,fill=color0,postaction={pattern=north east lines, pattern color=white}] (axis cs:6,0) rectangle (axis cs:6.35,14.8933852144775);
\draw[draw=white,fill=color1] (axis cs:1.65,3.94728505352314) rectangle (axis cs:2,16.8117182970047);
\draw[draw=white,fill=color1] (axis cs:2.65,6.36479051401184) rectangle (axis cs:3,31.0361406087875);
\draw[draw=white,fill=color1] (axis cs:3.65,15.8168080606781) rectangle (axis cs:4,49.1873962538583);
\draw[draw=white,fill=color1] (axis cs:4.65,16.0077846008701) rectangle (axis cs:5,41.0590175390244);
\draw[draw=white,fill=color1] (axis cs:5.65,16.0836326358871) rectangle (axis cs:6,38.5202422142029);
\draw[draw=white,fill=color1,postaction={pattern=north east lines, pattern color=white}] (axis cs:2,3.25712248885516) rectangle (axis cs:2.35,3.71941258907318);
\draw[draw=white,fill=color1,postaction={pattern=north east lines, pattern color=white}] (axis cs:3,7.35013741120685) rectangle (axis cs:3.35,9.75366756916046);
\draw[draw=white,fill=color1,postaction={pattern=north east lines, pattern color=white}] (axis cs:4,10.8525044935771) rectangle (axis cs:4.35,16.4409710679735);
\draw[draw=white,fill=color1,postaction={pattern=north east lines, pattern color=white}] (axis cs:5,14.5752984466611) rectangle (axis cs:5.35,19.9371992647648);
\draw[draw=white,fill=color1,postaction={pattern=north east lines, pattern color=white}] (axis cs:6,14.8933852144775) rectangle (axis cs:6.35,19.1608122587204);
\draw[draw=white,fill=color2] (axis cs:1.65,16.8117182970047) rectangle (axis cs:2,16.8117182970047);
\draw[draw=white,fill=color2] (axis cs:2.65,31.0361406087875) rectangle (axis cs:3,31.0361406087875);
\draw[draw=white,fill=color2] (axis cs:3.65,49.1873962538583) rectangle (axis cs:4,49.1873962538583);
\draw[draw=white,fill=color2] (axis cs:4.65,41.0590175390244) rectangle (axis cs:5,41.0590175390244);
\draw[draw=white,fill=color2] (axis cs:5.65,38.5202422142029) rectangle (axis cs:6,38.5202422142029);
\draw[draw=white,fill=color2,postaction={pattern=north east lines, pattern color=white}] (axis cs:2,3.71941258907318) rectangle (axis cs:2.35,4.0732531785965);
\draw[draw=white,fill=color2,postaction={pattern=north east lines, pattern color=white}] (axis cs:3,9.75366756916046) rectangle (axis cs:3.35,10.6340378284454);
\draw[draw=white,fill=color2,postaction={pattern=north east lines, pattern color=white}] (axis cs:4,16.4409710679735) rectangle (axis cs:4.35,17.4133159092494);
\draw[draw=white,fill=color2,postaction={pattern=north east lines, pattern color=white}] (axis cs:5,19.9371992647648) rectangle (axis cs:5.35,21.6864346265793);
\draw[draw=white,fill=color2,postaction={pattern=north east lines, pattern color=white}] (axis cs:6,19.1608122587204) rectangle (axis cs:6.35,21.6332992315292);

\end{groupplot}
\end{tikzpicture}
     \caption{A breakdown of the run times of \texttt{ADAPTIVE} and \texttt{INFORMED} into the constituent phases. We average the time in each phase over all problems which both algorithms solve successfully.}

        
         

    \label{fig:runtime_breakdown}
\end{figure}
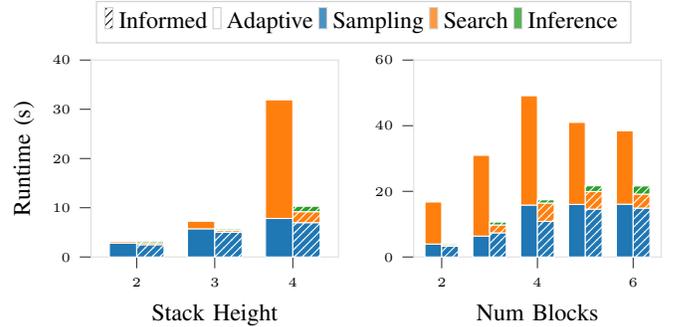



\underline{\textit{Effect of Additional Training Data.}} To explore to what extent more training data would lead to improved performance, we evaluated models trained on increasing subsets of a larger training set. We plot solution rates and run time over training set sizes up to 400 problems in figure \ref{fig:traindata-vs-solve}. We found that planning performance improves sharply, but improvements plateau after around 100 training problems.

\begin{figure}[h!]
\centering
\begin{tikzpicture}

\definecolor{color0}{rgb}{0.12156862745098,0.466666666666667,0.705882352941177}
\definecolor{color1}{rgb}{1,0.498039215686275,0.0549019607843137}
\pgfplotsset{every tick label/.append style={font=\tiny}}
\begin{groupplot}[group style={group size=2 by 1}]

\nextgroupplot[
legend cell align={right},
legend columns = 2,
legend style={fill opacity=0.8, draw opacity=1, text opacity=1, draw=white!80!black, nodes={scale=0.75, transform shape}, at={(1,1.3)},anchor=north, outer sep=0pt},
tick align=outside,
tick pos=left,
label style = {font=\scriptsize},
x grid style={white!69.0196078431373!black},
xlabel={Training Problems},
xmin=-15.8, xmax=419.8,
xtick style={color=black},
y grid style={white!69.0196078431373!black},
ylabel={Percent Solved},
ymin=16.46, ymax=60,
ytick style={color=black},
tick style = {font=\scriptsize},
width=4.5cm
]
\addplot [semithick, color0, mark=*]
table {%
4 19
40 26
100 54
200 54.5
320 55.6
400 55.8
};
\addlegendentry{Stacking}
\addplot [semithick, color1, mark=*]
table {%
4 18.4
20 42.6
40 46.2
100 54
320 55
400 59
};
\addlegendentry{Clutter}
\nextgroupplot[
tick align=outside,
tick pos=left,
x grid style={white!69.0196078431373!black},
xlabel={Training Problems},
xmin=-15.8, xmax=419.8,
xtick style={color=black},
y grid style={white!69.0196078431373!black},
label style = {font=\scriptsize},
tick style = {font=\scriptsize},
ylabel={Mean Solve Time},
ymin=43.4553860668659, ymax=77.3888715663433,
ytick style={color=black},
width=4.5cm
]
\addplot [semithick, color0, mark=*]
table {%
4 74.14
40 68.85
100 54.00
200 53.85
320 45.53
400 47.03
};

\addplot [semithick, color1, mark=*]
table {%
4 75.84
20 56.89
40 53.00
100 50.12
320 44.99
400 45.62
};

\end{groupplot}

\end{tikzpicture}
\caption{Effect of additional training data on planning performance in Stacking and Clutter tasks.}
\label{fig:traindata-vs-solve}   
\end{figure}
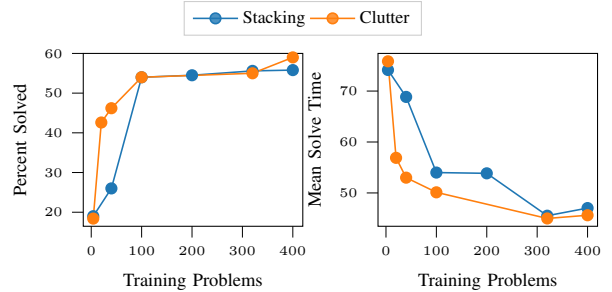

\underline{\textit{Multi-task Model.}} Since all the blocks world tasks share a single domain definition (i.e. predicates, actions and streams), we additionally experimented with training a single model on problem instances from all the tasks. We then evaluate this model on held-out instances from each problem type. We compare the performance of this model to that of the domain-specific models we evaluated in section \ref{sec:Experiments}. The performance of the combined model is usually on-par with the task-specific models, except in Sorting and Stacking, where the domain-specific model significantly outperforms.
\begin{table}[h!]
\tabcolsep=0.08cm
\centering
\caption{Solve percentage on held-out test problems for single-task vs multi-task models across Blocks World problem types.}
\begin{tabular}{c|cc|cc}
\toprule
& \multicolumn{2}{c|}{\textbf{ Single-Task }} & \multicolumn{2}{c}{\textbf{ Multi-Task}} \\
\midrule
\textbf{Domain} & \textbf{Solved \%} & \textbf{Time (s)}& \textbf{Solved \%} & \textbf{Time (s)} \\  \midrule
\textbf{Clutter      } & 54 &12.15&           \textbf{56} &    14.23  \\
\textbf{Non Monotonic} &  58 &33.88&         \textbf{65} &    35.69 \\
\textbf{Sorting      } & \textbf{77} & 21.07&          68 &    16.71 \\
\textbf{Stacking     } &  \textbf{54} & 7.94&        45 &    13.17 \\
\textbf{Distractors  } &  \textbf{100} &9.65&         99 &    17.90 \\
\bottomrule
\end{tabular}
\label{tab:combined}
\end{table}
 
\underline{\textit{Real-robot experiment.}} Real-robot experiment. In addition to the experiments described above, We deployed the solutions from our planner  on a Franka-Emika Panda robot as shown in Fig. 1. For these  experiments, the relevance model  is trained on stacking tasks in simulation,  but we execute and evaluate the trajectories it finds directly on the robot arm in an open-loop fashion.  


\section{Conclusion}
In this paper, we proposed a method for learning to search in Task and Motion Planning (TAMP) based on past experience computing successful task and motion plans in similar environments. A key component of our method is a learned, problem-conditioned ranking model that outputs the relevance of task-level subgoals and logical facts, which offers betters guidance in terms of task planning. We proposed the \texttt{INFORMED} algorithm which uses this model to order its expansion of streams with the goal of accelerating planning times for unseen problems. Our experiments demonstrate the effectiveness of this approach in the context of TAMP on a variety of problem types, where it leads to a significant speedup over existing methods, especially in problems with high branching factor, and including on a real robot manipulation setting.



\bibliographystyle{IEEEtran}
\bibliography{ICRA2022}

\newpage
\pagestyle{empty}
\appendix
\section{Supplementary Materials}
\subsection{Completeness of the \texttt{INFORMED} algorithm}
The algorithm described above inherits many of the same theoretical guarantees as those described in \cite{garrett2020pddlstream}. In particular, the proof of semi-completeness for optimistic algorithms (i.e. theorem 3) can be extended to the \texttt{INFORMED} algorithm under some assumptions about the relevance scores.

At a high level, the proof for theorem 3 argues that if $\tilde{l}$ is the minimum level at which a solution exists, a finite number of iterations will elapse before the algorithm considers all stream instances with level $l \leq \tilde{l}$, at which point the solution must be found.

The main point of departure between these optimistic algorithms and \texttt{INFORMED} is the use of predicted relevance scores, rather than level (equation \ref{eq:1}) for ordering the instantiation of streams. Therefore, to extend the argument to \texttt{INFORMED}, we simply need to show that there is a finite number of stream instances with a score greater than $y$ for any $y > 0$.


\begin{theorem}
The \texttt{INFORMED} algorithm is semi-complete under the following assumptions: (1) that the score that $\mathcal{M}_{\theta}$ assigns to a stream is strictly less than the scores assigned to its parents, and (2) that the relevance scores assigned to successive evaluations of the same stream instance are monotonically decreasing.
\end{theorem}

\begin{proof}
Let $\pi$ be a feasible solution to a given PDDLStream problem. Then $\text{PREIMAGE}(\pi)$ is the set of stream-certified facts which must be added the initial conditions $\mathcal I$ so that $\texttt{SEARCH}$ may find the correct optimistic solution. Let $y^*$ be the minimum score assigned to any of the elements of $\text{PREIMAGE}(\pi)$.

We can show that under conditions (1) and (2), a stream with score $\geq y^*$ has a maximum possible level, as defined in equation \ref{eq:1}.
To see this, let $\epsilon_{child} > 0$ be the minimum difference of scores between a stream and its child, and let $\epsilon_{eval} > 0$ be the minimum difference between successive evaluations of the same stream. We can bound the maximum level of a stream with score $y^*$ as $\text{max}_l(y^*) \leq 1 + \lceil\frac{1 - y^*}{\epsilon_{eval}}\rceil + \text{max}_l(y^* + \epsilon_{child}) \le \frac{1-y^*}{e_{child}} (1 + \frac{1-y^*}{e_{eval}})$.

\texttt{INFORMED} will instantiate the stream with the highest score at every iteration. By the argument above, every stream with score $\geq y^*$ must have level $\leq \text{max}_l(y^*)$. By theorem 3, we know that the number of stream instances with level $\leq max_l(y^*)$ is finite. Therefore, we conclude that a finite number of iterations will elapse before the algorithm considers all stream instances with score $\geq y^*$, at which point a solution must be found.

\end{proof}

In practice we enforce condition (2) by decaying the relevance score by a factor $\gamma < 1$ after every evaluation. Condition (1) is enforced by defining the score of a stream $s$ as: $$score(s) = M_\theta(s) \cdot \underset{\hat s \in s.parents}{\text{min}} score(\hat s)$$
During training, we approximate the min with a weighted average of the parent scores, where the weights are computed as a softmax over the negative scores.

\subsection{Model Architecture and Hyperparameters}
In section \ref{sec:relevance}, we have outlined the representation of planning problems as a relational graph, augmented with features that encode the 3D position of each object. This Problem Graph is the input to a GNN, composed of 3 graph network (GN) blocks  \cite{battaglia2018relational}. Each GN block is itself composed of a node model $\phi_v$, and an edge model $\phi_e$ which share the same architecture in table \ref{tab:node_edge_model}.

\begin{table}[h]
    \centering
    \begin{tabular}{c|c|}
        hidden size & 64  \\
        output size & 64  \\
        number of layers & 2  \\
        hidden activation & LeakyRelu \\
        output activation & Linear \\
    \end{tabular}
    \caption{Hyperparameters of node and edge encoders $\phi_{[v/e]}^{[1/2/ 3]}$ which comprise each of 3 GN blocks. }
    \label{tab:node_edge_model}
\end{table}

While this GNN produces embeddings of both its nodes and its edges as output, we only use the node embeddings, which serve as the inputs to the stream MLPs. Each stream MLP has a different number of inputs $I_{s}$ and outputs $O_{s}$ (depending on the domain definitions), however each of the input and output objects will be represented by a fixed size embedding. We detail the architecture of each MLP in \ref{tab:mlp_model}.

\begin{table}[h]
    \centering
    \begin{tabular}{l|c|c|c|}
        & $M_{s}^{encoder}$ & $M_{s}^{scorer}$ & $M_{s}^{decoder}$ \\
        \midrule
        input size & $64 \cdot I_{s}$  & 64 & 64\\
        number of layers & 1 & 2 & 2  \\
        hidden size & 64 & 64 & 64\\
        hidden activation & - & LeakyRelu & LeakyRelu \\
        output size & $64$  & 1 & $64 \cdot O_{s}$\\
    \end{tabular}
    \caption{Architecture and hyperparameters for each stream MLP. Note that the size of the input to the encoder, and output of the decoder depend on the stream's domain definitions. }
    \label{tab:mlp_model}
\end{table}


\subsection{Problem Distribution Details}
\begin{table}[h!]
\caption{To supplement the descriptions of environments given in section \ref{sec:Experiments}, we provide details of object numbers and placement distributions for the train/test sets.}
\begin{tabular}{lll}
\toprule
\multicolumn{1}{c}{}   & \multicolumn{1}{c}{\textbf{\begin{tabular}[c]{@{}c@{}}Number of Objects\\ Train / Test\end{tabular}}} & \multicolumn{1}{c}{\textbf{Placement Distribution}}                               \textbf{} \\ \midrule
\textbf{Clutter}       & \begin{tabular}[c]{@{}l@{}}Blocks: 2-4 / 2-6\\ Blockers: 4-8 / 4-12\end{tabular}                      & Poisson Disc                                                                                   \\ \hline
\textbf{Non Monotonic} & \begin{tabular}[c]{@{}l@{}}Blocks: 1-3 / 2-6\\ Blockers: 1-3 / 2-6\end{tabular}                       & \begin{tabular}[c]{@{}l@{}}Uniform\\ Adjacent\end{tabular}                                     \\ \hline
\textbf{Sorting}       & \begin{tabular}[c]{@{}l@{}}Blocks: 2-7 / 2-10\\ Blockers: 2-7 / 2-10\end{tabular}                     & Poisson Disc                                                                                  \\ \hline
\textbf{Stacking}      & \begin{tabular}[c]{@{}l@{}}Blocks: 2-4 / 2-7\\ Blockers: 0 / 0\end{tabular}                           & Uniform                                                                                       \\ \hline
\textbf{Distractors}   & \begin{tabular}[c]{@{}l@{}}Blocks: - / 2-3\\ Blockers: - / 10-50\end{tabular}                         & \begin{tabular}[c]{@{}l@{}}Uniform(red/blue tables)\\ Uniform(green/purple tables)\end{tabular}          \\ \hline
\textbf{Kitchen}       & \begin{tabular}[c]{@{}l@{}}Cabbage: 1-3 / 1-4\\ Radish: 0-3 / 0-3\\ Glass: 0-2 / 0-2 \\ Goals: 1-7 / 1-9\end{tabular}     & Uniform               \\\bottomrule  \bottomrule
\end{tabular}

\end{table}
\end{document}